\documentclass[11pt]{article}

\usepackage[utf8]{inputenc}
\def\final{1}
\usepackage{color}
\usepackage{subfigure}
\usepackage{url}
\usepackage{graphicx}
\usepackage{amsmath}
\usepackage{amsthm}
\usepackage{bm}
\usepackage{algorithm}
\usepackage{algorithmic}
\usepackage[plainpages=false]{hyperref}
\usepackage{cleveref}

\crefformat{footnote}{#2\footnotemark[#1]#3}

\usepackage{titlesec}
\titleformat{\section}
{\normalfont\Large\bfseries\filcenter}{\thesection.}{1 ex}{}
\titleformat{\subsection}
{\normalfont\normalsize\bfseries}{\thesubsection.}{1 ex}{}
\titleformat{\subsubsection}[runin]
{\normalfont\normalsize\bfseries\filcenter}{\thesubsubsection.}{1 ex}{}

\usepackage[charter]{mathdesign}
\usepackage[mathcal]{eucal}

\usepackage[margin=1.2 in]{geometry}

\usepackage[square,numbers,sort&compress]{natbib}

\ifnum\final=0
\newcommand{\mynote}[1]{\marginpar{\tiny\sf #1}}
\else
\newcommand{\mynote}[1]{}
\fi

\usepackage{thmtools,thm-restate}
\declaretheoremstyle[qed=$\diamond$,headpunct={ --- },headfont=\normalfont\itshape]{myremark}
\declaretheoremstyle[bodyfont=\normalfont]{mydefinition}
\declaretheorem[name=Theorem]{Thm}
\declaretheorem[within=section,name=Lemma]{Lem}

\declaretheorem[sibling=Lem,name=Definition,style=mydefinition]{Def}
\declaretheorem[sibling=Lem,name=Example,style=mydefinition]{Ex}

\declaretheorem[sibling=Lem,name=Notation,style=mydefinition]{Not}

\newcommand{\compresslist}{\setlength{\itemsep}{1pt}
\setlength{\parskip}{0pt}
\setlength{\parsep}{0pt} }

\renewcommand{\eqref}[1]{({\ref{eq:#1}})}

\newcommand{\iprod}[2]{\left\langle {#1}\, , \, {#2}\right\rangle}

\newcommand{\C}{\mathbb{C}}
\newcommand{\K}{\mathbb{K}}
\newcommand{\N}{\mathbb{N}}
\newcommand{\R}{\mathbb{R}}

\newcommand{\calF}{\mathcal{F}}

\newcommand{\calI}{\mathcal{I}}

\newcommand{\calS}{\mathcal{S}}

\newcommand{\mX}{\mathbf{X}}

\newcommand{\va}{\mathbf{a}}
\newcommand{\vb}{\mathbf{b}}

\newcommand{\rk}{\operatorname{rank}}
\newcommand{\lspan}{\operatorname{span}}

\newcommand{\diag}{\operatorname{diag}}

\newcommand{\rowspan}{\operatorname{rowspan}}

\DeclareMathOperator*{\Id}{I}
\DeclareMathOperator*{\Van}{V}

\begin{document}

\title{Dual-to-Kernel Learning with Ideals}

\author{Franz J. Király \thanks{Department of Statistical Science, Univerity College London, and MFO
\url{f.kiraly@ucl.ac.uk}} \and
Martin Kreuzer \thanks{Universität Passau, \url{Martin.Kreuzer@uni-passau.de}} \and
Louis Theran\thanks{Inst. Math., AG Diskrete Geometrie, Freie Universität Berlin, \url{theran@math.fu-berlin.de}}}

\date{}
\maketitle

\begin{abstract}
In this paper, we propose a theory which unifies kernel learning and symbolic algebraic methods. We show that both worlds are
inherently dual to each other, and we use this duality to combine the structure-awareness of algebraic methods with the efficiency and generality
of kernels. The main idea lies in relating polynomial rings to feature space, and ideals to manifolds, then exploiting this generative-discriminative duality on kernel matrices. We illustrate this by proposing two algorithms, IPCA and AVICA, for simultaneous manifold and feature learning, and test their accuracy on synthetic and real world data.
\end{abstract}

\section{Introduction}\label{Sec:intro}
In this paper, we propose a learning theory which is the {\bf synthesis of kernel and symbolic algebraic methods}, by exposing inherent dualities between them. We use this duality to combine the structure-awareness of algebraic methods with the efficiency and generality of kernels. Since their invention by Boser, Guyon and Vapnik~\cite{BoserVapnik92,Vapnik95}, {\bf kernel methods} have had a fundamental impact on the fields of statistics and machine learning. The major appeal of using kernel methods for learning consists in using the kernel trick, first proposed by Aizerman, Braverman and Rozonoer~\cite{Aizerman64}, which allows to make otherwise costly computations in the feature space implicit and thus highly efficient for a huge variety of learning tasks - see e.g.~\cite{Scholkopf02, Shawe-Taylor04}. However, the major advantage of kernel methods is also their major drawback: since kernels implicitize feature space computations, the learnt model is implicit as well; in most scenarios, though kernels perform excellently, it is indeed a principal open question what it is that kernels learn - \emph{a question to which we can provide an answer through duality with ideals}. On the other hand, {\bf symbolic-algebraic methods} are inherently structural, as they yield explicit and compact representations of the data, as so-called ideals, with the major advantage of being directly interpretable. The seminal Buchberger-Möller algorithm~\cite{Moller82} allows to transform one representation into a different, easier and sparser one. One major issue of the Buchberger-Möller algorithm is that it is numerically unstable and therefore not applicable in noisy scenarios - this has been addressed by a class of numerical algorithms surrounding the Approximate Vanishing Ideal (AVI) method~\cite{AVI,Sauer07}. While these algorithms offer attractive and explicit representations, the major issue with symbolic methods preventing broad applicability is their exponential (or higher) complexity, and model selection issues - \emph{which we can considerably reduce through dual kernelization}. In the intersection of symbolic algebra and kernel method, we propose general tools and two algorithms, IPCA and AVICA, which simultaneously can learn generative information from the data manifold and discriminative features; more generally, we argue that the kernel-ideal duality translates generative and discriminative tasks in the kernel world directly to discriminative and generative tasks in the algebra world, which allows to combine the advantages of either while avoiding the disadvantages of both. We therefore expect our findings to have a considerable impact on the fields
of learning, statistics, algebra, and the interaction between those.

\section{Ideal-Kernel-Duality}\label{Sec:ideals}

\subsection{The Polynomial Ring as Dual Kernel Space}
We introduce the main objects we are relating through duality.
We start by defining polynomial kernels, the main kernel-type objects involved
in the duality presented here. Later we will explain how to treat general kernels. In this paper, $\K$ will be one of the fields
$\R$ or $\C$.

\begin{Def}
Let $\theta \in (0,1)$ be a fixed real number.
Slightly different from the usual definition, we denote\\
by $k_d$ the homogenous polynomial kernel function\\
$k_d:\K^n\times \K^n\rightarrow \R,\; (x,y)\mapsto \theta^d\cdot
\langle x,y\rangle^d$, and\\
by $k_{\le d}$ the inhomogenous polynomial kernel function\\
$k_{\le d}:\K^n\times \K^n\rightarrow \R,\; (x,y)\mapsto (\theta\cdot \langle x,y\rangle+1)^d$.
\end{Def}
The usual convention for the kernel function is obtained after dividing
by $\theta^d$. Since $\theta$ is chosen arbitrarily in $(0,1)$, no qualitative
change is introduced by our convention. It is however, as we will show,
the more natural one. On the algebra side, the main objects linked via the duality are vector
spaces of polynomials.

\begin{Not}
We denote
by $\mX = (X_1,\ldots, X_n)$ a vector of coordinate variables,
by $\K[\mX]_{d}$ the $\K$-vector space of homogeneous polynomials of degree~$d$
in~$\mX$,
by $\K[\mX]_{\le d}$ the $\K$-vector space of all (homogeneous or inhomogeneous)
polynomials of degree at most~$d$ in~$\mX$,
by $\K[\mX] = \K[\mX]_{0}\oplus \K[\mX]_{1}\oplus \cdots$ the ring of all polynomials
in~$\mX$, and
by $\mX^\va$, for $\va \in \N^n$, the monomial $X_1^{\va_1}\cdot \dots \cdot X_n^{a_n} \in \K[\mX]_d$ where $d = a_1+\dots+a_n$.
\end{Not}

The dimension of $\K[\mX]_{\le d}$ is ${n+d \choose d}$,
and the dimension of $\K[\mX]_{d}$ is ${n+d-1\choose d}$. The dimension of $\K [\mX]$ is infinite. The ring $\K [\mX]$ is dual to the vector space of kernel decision functions in the following way.

\begin{Thm}\label{Thm:dual}
Let $d\ge 0$, let $m\ge \dim \K [\mX]_*$, where $*$ can denote $d$ or $\le d$, and
let $y_1,\dots, y_m\in \K^n$ be generic. Then we have $\K [\mX]_* = \lspan \{k_*(y_i,\mX), 1\le i\le m\}.$
\end{Thm}

\begin{proof}
Linear independence of up to $\dim \K [\mX]_*$ of the $k_*(y_i,\mX)$ follows from genericity of $y_i$. Since $\lspan \{k_*(y_i,\mX), 1\le i\le m\}\subseteq \K [\mX]_*$, this yields the claim.
\end{proof}

\subsection{Duality of Kernel and Ring Scalar Products}
Let $d\ge 0$, and let~$\ast$ denote~$d$ or $\le d$.
Theorem~\ref{Thm:dual} implies that $\K [\mX]_*$ is dual to the feature space
of the kernel~$k_*$. By passing to the limit, this implies that $\K [\mX]$
contains the dual of any feature space. Explicitly, this is seen as follows: consider the usual feature
map $\phi_*: \K^n\rightarrow \calF_*$, where $\calF_*$ is the feature space.
Elementary computations, such as in section~2.1 or problems~2.6.2-3 of~\cite{Scholkopf02}, show that the feature map
can be explicitly identified as
$\phi_*: x\mapsto (\gamma_\va\cdot x^\va\;:\;\va\in\N^n,\; a_1+\dots+a_n=*),$
with $\gamma_\va = \sqrt{\theta^{d}\cdot \binom{d}{\va}}$.
By counting the number of distinct $\va$, we see that $\dim\calF_* = \dim \K[\mX]_*$.
This can be made more explicit by interpreting polynomials in $\K[\mX]_*$
as elements of the dual of~$\calF_*$, i.e., as functions $\calF_*\rightarrow \K$.
Namely, for a polynomial $f\in\K[\mX]_*$ with $f=\sum_{\va\in\N^n} c_\va \mX^\va$,
write $f^\vee:=(c_\va/\gamma_\va,\va\in\N^n)$.
One checks that $\langle f^\vee, \phi_*(p)\rangle = f(p)$ for all $p\in\K^n$.
Thus we can identify~$\calF_*$ with the dual polynomial ring $\K[\mX]_*^\vee$.
Since the latter is finite dimensional, it is self dual, so there is a canonical identification $\calF_*^\vee\cong \K [\mX]_*$.

We use this identification to transfer the canonical scalar product on~$\calF_*$
to its dual~$\K[\mX]_*$ in the natural way. Namely, in order to be compatible
with~$\phi_*$, it needs to be the unique scalar product on~$\K[\mX_*]$
such that, for $x,y\in \K^n$, the reproducing property
$\iprod{k_*(x,\mX)}{k_*(\mX,y)} = k_*(x,y)$
holds.
To achieve this, the factor~$\gamma_\va$ occurring in~$f^\vee$ above must be corrected for.
Consequently, an explicit description of the scalar product on~$\K[\mX]_*$ is given
as follows: Fix $d\ge 0$, and let $\va,\vb\in \N^n$ be exponent vectors. Define $\iprod{\mX^\va}{\mX^{\vb}} = 0$ if $\va\neq \vb$, and
\begin{equation}\label{Eq:inner-product}
\iprod{\mX^\va}{\mX^\va} = \gamma_\va^{-2} = \theta^{-d}\cdot \binom{d}{\va}^{-1} = \theta^{-d}\cdot \frac{a_1!\cdot\dots\cdot a_n!}{d!},
\end{equation}
and then extend this bi- or sesquilinearly to all of~$\K[\mX]$.\\
Theorem~\ref{Thm:dual} then implies that, for $f,g\in\K[\mX]_*, p\in \K^n$,
the following equalities hold:
\begin{align}\label{Eq:eval-ip}
f(p) &= \iprod{f}{k_*(\mX, p)}\\
\iprod{fg}{k_*(\mX, p)^2} &=
\iprod{f}{k_*(\mX, p)}\iprod{g}{k_*(\mX, p)}\label{Eq:eval-prod}.
\end{align}
By applying Theorem~\ref{Thm:dual} and properties of the symmetric outer product, we get
an important consequence of Equation~\ref{Eq:eval-prod}, namely that the scalar
product is multiplicatively absorbing for orthogonality. That is, letting
$f_1,f_2,g_1,g_2\in \K [\mX]$ such that $\langle f_i,g_j\rangle = 0$ for all $i,j$,
we have
\begin{align}\label{Eq:orthdeg}
\iprod{f_1f_2}{g_1 g_2} &= 0.
\end{align}
Note that in the usual convention for the polynomial kernel, similar equalities are still valid, but much less concise to express.
The duality above is an algebraic analogue of the theory of reproducing kernel Hilbert spaces. The associated Hilbert space is the space of polynomial functions $f:\K^n\rightarrow \K$, which can be identified with $\K [\mX]_*$ by an additional dualization; phrased in algebraic terms, by replacing the evaluation homomorphism $f$ with the corresponding symbolic polynomial. Note that the equations~\ref{Eq:eval-ip} and~\ref{Eq:eval-prod} could also be obtained combining Riesz representation, or the reproducing property of $k_*$, with this identification, compare e.g.~section 2.2.3 in~\cite{Scholkopf02}.Equality~\ref{Eq:orthdeg}, on the other hand, is purely obtained from algebra. The next section will also go beyond what could be reached from usual RKHS duality alone.

\subsection{Ideals are Dual to Manifolds}
Next we show that ideals - a classical concept in algebra - are the proper objects
to dualize manifolds, in the same way as the polynomial ring dualizes feature space.

\begin{Def}\label{Def:ideal}
An \emph{ideal} is a linear subspace $\calI\subseteq \K[\mX]$ which also absorbs multiplication, that is, which satisfies $f\cdot g \in \calI$ for all $f\in\K[\mX],
g\in \calI$. For $d\ge 0$, we let $\calI_{\le d}:=\calI\cap \K[\mX]_{\le d}$.
\end{Def}

While an ideal is in general infinite dimensional as a $\K$-vector space,
Hilbert's basis theorem says that all ideals admit a \emph{finite} set of additive-multiplicative generators. As for vector spaces, such sets of generators need not be unique. An important class of ideals is given as follows.

\begin{Ex}\label{Ex:vanishing-ideal}
Let $\calS\subseteq \K^n$. Then the $\K$-vector space of polynomials
$\Id(\calS) := \{f\in \K[\mX] :\text{$f(s) = 0$ for all $s\in \calS$} \}$
is an ideal. It is called the \emph{vanishing ideal of $\calS$}.
\end{Ex}
While $\calS$ can be in principle any subset of~$\K^n$, we will be mainly
concerned with the case where~$\calS$ is a manifold. In this case,
the vanishing ideal $\Id (\calS)$ is the dual of the manifold $\calS$
in the following precise sense which is an analogue to Theorem~\ref{Thm:dual}.

\begin{Thm}\label{Thm:dual-manifold}
Let $\calS\subset \K^n$ be a manifold. It holds that
$
\Id(\calS)\cap\K[\mX]_* = \lspan \{k_*(s,\mX),\; s\in \calS\}^{\perp}.
$
\end{Thm}
\begin{proof}
This follows from the definition of~$\Id (\calS)$ and Equation~\ref{Eq:eval-ip}.
\end{proof}

Theorem~\ref{Thm:dual-manifold} relates manifolds to ideals via kernel duality.
Intuitively, it says that the manifold~$\calS$ corresponds to a linear subspace
of feature space. This is a kernelized version of the usual algebra-geometry duality
and and relates the discriminative description through decision functions $k_*(s,\mX)$ to the generative description by~$\calS$.

\subsection{Interpolation Space and Kernel Border Bases}
In this section we reveal a further duality between certain kernel matrices and ideals.
First we introduce the following notation for kernel matrices.

\begin{Not}
Let $X\in\K^{N\times n}$, let $Y\in \K^{D\times n}$, and let
$x_1,\dots, x_N, y_1,\dots, y_D$ be the rows of~$X$ and~$Y$, resp.
For a kernel $k:\K^n\times \K^n\rightarrow \R$, we denote by~$k(X,Y)$
the $(N\times D)$-matrix which has the number $k(x_i,y_j)$ as its entry
in position $(i,j)$.
\end{Not}

The concept of an interpolation space is motivated by the duality
in Theorem~\ref{Thm:dual-manifold} and is the orthogonal to an ideal.

\begin{Def}
Let $\calS\subseteq \K^n$, and let $\calI = \Id (\calS)$.
The \emph{interpolation space} of~$\calS$ is the vector space $\K [\calS]:=\calI^\perp$,
where the orthogonal is taken w.r.t.~the scalar product defined above.
Given $d\ge 0$, we also write $K[\calS]_{\le d}$ for $\calI_{\le d}^\perp$.
\end{Def}

Intuitively, the interpolation space gives a canonical basis for the space of
functions on~$\calS$, orthogonal to those vanishing on~$\calS$. It is a fixed choice of representatives for the factor ring $\K[\mX]/\Id(\calS)$ usually defined in algebra.
The results of the previous sections yield the following duality statements
between the interpolation space and polynomial kernel matrices.

\begin{Thm}\label{Thm:dual-interp}
Let $\calS\subseteq \K^n$, let $x_1,\dots, x_N\in\calS$ be generic, and
let $X\in \K^{N\times n}$ be the matrix which has the~$x_i$ as columns.
Let $Y\in\K^{D\times n}$ be generic. Denote $K_d:=k_{\le d}(X,Y)$. Then
the following claims hold:
\begin{description}
\compresslist
\item[(i)] $\rk K_d = \min\left(\dim \K [\calS]_{\le d}, N, D\right)$.
\item[(ii)] It holds that $k_{\le d}(\mX,Y)\cdot \rowspan K_d \subseteq \K [\calS]_{\le d}$.
Equality holds if and only if $N,D\ge \dim \K [\calS]_{\le d}$.
\item[(iii)] For $\alpha\in \K^D$, it holds that
$k_{\le d}(\mX,Y)\cdot \alpha\in \Id (\calS)$ only if $K_d\cdot \alpha = 0$.
The converse is true if and only if $N,D \ge \dim \K [\calS]_{\le d}$.
\end{description}
(Notice the difference between the $n$-tuple of variables $\mX$ and the data matrix~$X$.)
\end{Thm}

\begin{proof}
For $d\ge 0$, we write $\tau:=\dim \K [\calS]_{\le d}$. Claim (i): First assume that $N,D\ge \tau$. Theorem~\ref{Thm:dual-manifold}
implies that  $\lspan \{k_{\le d}(s,\mX),\; s\in \calS\} =
(\Id(\calS)\cap\K[\mX]_{\le d})^\perp = \K [\calS]_{\le d}$. In particular, this
shows $\dim \lspan \{k_{\le d}(s,\mX),\; s\in \calS\} = \tau$.
Therefore $D\ge \tau$ generic elements of the form $k_{\le d}(\mX,y_i)$
will generate $\K [\calS]_{\le d}$, and their span will have dimension $\tau$.
By interpreting the variables in~$\mX$ once more as functions
$\K^n\rightarrow \K^n$, we view $k_{\le d}(\mX,y_i)$ as a function
$k_{\le d}(.,y_i):\K^n\rightarrow \R$. Then the polynomials $k_{\le d}(\mX,y_i)$
span a vector space of dimension~$\tau$ if and only if the functions
$k_{\le d}(.,y_i)$ do that as well. By substituting $N\ge \tau$ generic arguments
$x_1,\dots, x_N$, we see that the vectors $k_{\le d}(X,y_i)$ span a vector
space of dimension~$\tau$. This is equivalent to $\rk K_d = \tau$, proving the statement in case $N,D\ge \tau$. The general statement follows by starting
with $N,D=\tau$ and removing rows and columns. Claims (ii) and (iii) follow from the fact that, if $N\ge \tau$,
the vector of variables~$\mX$ in the conditions can be equivalently replaced by~$X$.
The case $N\lneq \tau$ follows again by removing rows and columns.
\end{proof}

Theorem~\ref{Thm:dual-interp} states that a large enough kernel matrix
of type $K(X,Y)$ contains all information on~$\calS$, assuming the kernel degree is high enough as well. Algorithmically, it yields the important statement that, instead of size $O(\dim \K[\mX]_{\le d})=O(n^d)$ matrices which grow exponentially in $d$, we have
to deal with size $O(\dim \K [\calS]_{\le d})= O(N)$ matrices instead. These are
matrices whose size is bounded from above by the number of data points.
In fact, the effective size is usually even lower, depending on model complexity.
We also note that claims~(ii) and~(iii) of Theorem~\ref{Thm:dual-interp} make
statements about functions of the form $F(.)=\sum_{i=1}^D\alpha_i k(.,y_i)$, which is a familiar form of decision functions. Since the functions in claim~(ii) are part of
the interpolation space of~$\calS$, we call them "discriminiative features".
The functions in claim~(iii) vanish on~$\calS$. Therefore they carry structural
information about~$\calS$, so we will call them "generative features".
A basis of~$\Id (\calS)$ consisting of such generative features will be called
a \emph{kernel border basis}. These features will be identified by both algorithms we introduce in the subsequent section.

\section{Learning with Ideals}\label{Sec:kernels}

Similar to the ubiquity of kernels, ideals enable us to address a wide variety of
learning scenarios. The general motive in ideal-manifold duality is that generative
features are transformed into discriminative ones and vice versa. For instance,
a strategy  for estimating a discriminative function in kernel learning, e.g., by
kernel SVM, will be close to estimating a descriptive feature in learning a manifold
via obtaining ideal generators. Conversely, estimating descriptive features
such as in kernel PCA will relate to obtaining discriminative features from
the interpolation space. This permits us to "dualize" techniques and to
transfer the absorption property of ideals back into the kernel world,
thus yielding new model selection tools and compact representations in terms
of kernel degree.

\subsection{Statistical Learning Theory for Ideals}\label{sec:kernels.learning}
The following learning approach to ideals is inspired by Vapnik's statistical
learning theory. We assume that there is a \emph{generative truth}, modelled
by an unknown algebraic manifold $\calS\subset \K^n$ with vanishing ideal
$\calI = \Id(\calS)$. The sampling process produces a discrete point
set $x_1,\ldots, x_N\in \K^n$, where $x_i = s_i + \varepsilon_i$,
with $s_i\in \calS$ sampled from a (Hausdorff-)continuous density
on $\calS$ and $\varepsilon_i\in\K^n$ being i.i.d.~centered noise of finite variance.
Basic learning tasks can be expressed in this ideal-learning framework as follows:

\begin{Ex}[name={Dimension reduction}]\label{Ex:dimension-reduction}
The dimension $n$ is large.  The task is to estimate the true
manifold $\calS$, assuming that $d\ll n$.
\end{Ex}

\begin{Ex}[name={Regression}]\label{Ex:regression}
The variables $\mX$ are partitioned into dependent and independent variables.
The noise acts only on the dependent variables, and the task is to estimate~$\calS$.
\end{Ex}

\begin{Ex}[{name={Classification}}]\label{Ex:classification}
The generative truth $\calS$ is assumed to have irreducible components
$\calS_1,\ldots,\calS_t$. The sample~$S$ is given as points $(x_i,\ell_i)$
with labels $\ell\in[t]$.  The task is to estimate
the components ~$\calS_i$ of~$\calS$.
\end{Ex}

\begin{Ex}[name={Clustering}]\label{Ex:clustering}
The generative truth~$\calS$ is assumed to have irreducible components
$\calS_1,\ldots,\calS_t$.  The sample~$S$ is unlabeled, and the task
is to estimate $\calS_1,\ldots,\calS_t$.
\end{Ex}

Measures of statistical optimality will be given after presenting our main algorithmic principle.

\subsection{A New Manifold Learning Kernel PCA}
Now we describe the basic idea behind computing with kernel duals.
It can be applied in all of the examples above. Algorithm~\ref{Alg:features},
which we term Ideal PCA, takes data $x_1,\dots, x_N$, sampled with noise
from a manifold~$\calS$, and returns feature functions of the
form $F(.) = \sum_{i=1}^D \alpha_i k(.,y_i)$ which are labelled either
generative (= part of the interpolation space) or discriminative (= part of
the kernel border basis). The kernel function~$k$ is one of the polynomial
kernels, but could be any kernel in principle.

\begin{algorithm}[h]
\caption[\texttt{IPCA} Computes interpolation space and kernel border basis]{\texttt{IPCA} Computes interpolation space and kernel border basis.\\
\textit{Input:} data $x_1,\dots x_N\in \K^n$, given as rows of an $(N\times n)$ matrix $X$, degree $d$, feature space size $D$, threshold $\epsilon$.\\
\textit{Output:} Feature functions $F_i = \sum_{j=1}^D \alpha_{ij} k(y_j,.)$, with labels "generative" or "discriminative" and quanta $\sigma_i$. \label{Alg:features}}
\begin{algorithmic}[1]
\STATE Sample $D$ random points $y_i\in \K^n$; write those into a $(D \times n)$
matrix $Y$.
\STATE Compute the $N\times D$ kernel matrix $K=k(X,Y)$ with $(i,j)$-th
entry $k(x_i, y_j)$.
\STATE Compute the singular value decomposition $K = U S V^\top$,
$S=\diag (\sigma_1,\dots,\sigma_m$)
\STATE The $\alpha_{ij}$ in the output are the entries of~$V$; the $y_j$
are the $y_j$ sampled above.
\STATE For each feature $F_i$, assign the label "discriminative"
if  $\sigma_i\ge \epsilon$, otherwise "generative". Also return the $\sigma_i$.
\end{algorithmic}
\end{algorithm}
The number $D$ should be chosen sufficiently large, either as $D\ge N$, or $D \ge \dim \K [\calS]_{\le d}$, if known. In these cases, Theorem~\ref{Thm:dual-interp} guarantees convergence in the noiseless case; we will see later that this also makes IPCA (Alg.~\ref{Alg:features}) a noise consistent algorithm. In the output, the discriminative features are expected to vary strongly when leaving the manifold~$\calS$. Therefore
they describe the internal structure of the data. On the other hand, the generative
features are expected to almost vanish in a neighborhood of~$\calS$. Therefore they
describe the manifold itself. The singular values~$\sigma_i$ yield a quantitative
measure. All - or some of the - features obtained from IPCA (Alg.~\ref{Alg:features})
can be bulk evaluated in an efficient way: first compute the kernel matrix $K= k(X,Y)$,
then the features are obtained from the matrix $K\cdot V$. We remark again that it is not necessary to choose $D = O(n^d)$ - as it would be in symbolic methods - due to the rank guarantee in Theorem~\ref{Thm:dual-interp}.

\subsection{Noise Consisteny}

As already said above, IPCA (Alg.~\ref{Alg:features}) transfers the exact statement
in Theorem~\ref{Thm:dual-interp} to the case of noisy data. We now show that IPCA
does this in a beneficial way. First we want to remark that the classical concept
of consistency will not be applicable here, since a manifold~$\calS$ which is a
point can lead to the same observations as a line~$\calS$ if i.i.d.~Gaussian noise
is added to the samples. Therefore, no algorithm can ``converge'' to~$\calS$
in the limit of the sample size. We argue that the proper notion of consistency in the manifold setting is \emph{noise consistency}:

\begin{Def}\label{Def:stability}
Consider the learning setup outlined in section~\ref{sec:kernels.learning}. That is,
let $\calS\subset \K^n$ be manifold with ideal $\calI = \Id(\calS)$, and assume that
we have noisy samples $x_i = s_i + \varepsilon_i$ from~$\calS$. We say that
an estimator $\widehat{\calS}$ is \emph{noise consistent} if
$\Id(\widehat{\calS}) =: \widehat{\calI} \to \calI$ as
$||(\varepsilon_1,\ldots,\varepsilon_n)||\to 0$, where convergence
of $\widehat{\calI}$ is defined as convergence of all vector spaces
$\widehat{\calI}_{\le d} \to \calI_{\le d}$ (possibly of different order).
\end{Def}
Intuitively, noise consistency is the combination of stability with respect to noise, and correctness in the noise-free case.

\begin{Thm}\label{Thm:basic-estimate}
Consider the learning setup outlined in section~\ref{sec:kernels.learning}. That is,
let $\calS\subset \K^n$ be a manifold with ideal $\calI = \Id(\calS)$, and assume that
we have noisy samples $x_i = s_i + \varepsilon_i$ from $\calS$.
IPCA (Alg.~\ref{Alg:features}) estimates~$\calS$ noise-consistently in the following
sense: assume $\calI$ is generated in degree~$d$ or less,
let $N\ge \K[\calS]_{\le d}$, let $\widehat{\K [\calS]_{\le d}}$ be the vector space
generated by the discriminative features that IPCA, with inhomogenous
kernel $k=k_{\le d}$, outputs. Let $\widehat{\calI}$ be the ideal generated by the
orthogonal complement of $\widehat{\K [\calS]_{\le d}}$. Then, $\Van(\widehat{\calI})$ is a noise consistent estimate for $\calS$.
\end{Thm}

\begin{proof}
If $d=1$, this follows directly from the Eckart-Young-Theorem which implies that thresholded SVD is a noise consistent estimator for the span of a matrix. The general statement is implied as follows: by Theorem~\ref{Thm:dual-interp} and noise consistency of SVD, $\widehat{\K [\calS]_{\le d}}$ is a noise consistent estimate for $\K [\calS]_{\le d}$. Thus, by passing to the orthogonal complement, $\widehat{\calI}_{\le d}$ is a noise consistent estimate for $\calI_{\le d}$. Since $\calI$ is generated in degree $d$, this implies that $\widehat{\calI}$ is a noise consistent estimate for $\calI$, which implies the statement.
\end{proof}

\subsection{Informal Analysis of the Basic Method}\label{sec:kernels.analysis}
At first glance, IPCA (Alg.~\ref{Alg:features}) may seem to be another version of
kernel PCA or kernel SVD. However, there is one main difference: the matrix~$Y$
is random. Therefore we do not work with the kernel matrix $k(X,X)$, as usual,
but with a matrix $k(X,Y)$.\\
This enables us to look at the feature span of~$X$ \emph{from the outside},
whereas the classical approach only looks at relations between~$X$ and~$X$.
More specifically, doing PCA or SVD or any method involving only $k(X,X)$
will reveal only features \emph{inside the data manifold}. The manifold
itself - as the important generative object - will not be identified.
This shortcoming has already been noticed in~\cite{VCA2013}[section 3.1: "kernels
can't help"], where the authors conclude that such methods are not suitable
for learning generative features in an algebraic setting. The matrix $k(X,Y)$
can be used to capture the \emph{extrinsic structure} of the data manifold.\\
What happens mathematically can be exposed using a linear (and noise-free) example:
take the kernel to be the linear scalar product $k=\langle.,.\rangle$.
Suppose our data $x_1,\dots, x_N$ come from a linear subspace $L\subseteq \R^n$.
Then we would like to learn discriminative features, that is, features that
vary among the $x_i$, here the principal components, and generative features,
in this simple case the subspace~$L$ from which the data are sampled. Writing~$X$
as the $(N\times n)$-matrix with the $x_i$ as rows, the kernel matrix is
the $(N\times N)$-matrix $K = k(X,X) = X X^\top$. Note that this differs
from the $(n\times n)$-matrix $X^\top X$ which is taken in classic PCA (for
centered data). Singular value decomposition of~$K$ will reveal features of~$X$,
such as the dimension of~$L$, through the rank of~$K$. However, rotating~$L$ together
with the~$x_i$, or embedding it into a different $\K^{n'}$ will leave~$K$ unchanged.
Therefore the generative information on~$L$ is lost in $k(X,X)$ since this matrix
contains only information on~$X$ \emph{inside}~$L$. On the other hand,
if a random matrix $Y \in \R^{D\times n}$ is taken with $D\ge n$, and if
$K'= k(X,Y) = X Y^\top$ is considered, the space~$L$ can easily be reconstructed
from the right singular vectors. Moreover, all information on $k(X,X)$ is potentially contained as well, most easily (but impractically) by adding in the rows of~$X$
as rows of~$Y$. This also shows that the important part of~$Y$ is that "orthogonal"
to~$X$, because it allow us to capture the \emph{extrinsic} structure of~$L$.\\
The case of general kernels is analogous, if we replace the scalar products above
by the kernel function. The role of the~$y_i$, which are above a basis for $\K^n$,
is played by the feature vectors $\phi (y_i)$ which now span the complete feature
space. The mathematical justification is given by Theorem~\ref{Thm:dual-manifold}
which shows that the manifold $\calS$ corresponds to a proper linear subspace
of the complete feature space. The interpolation space is exactly orthogonal to
decision functions $k(x_i,\mX)$ with~$x_i$ a data point. Therefore $k(X,X)$ cannot
be used to say anything about the interpolation space. On the other hand, Theorem~\ref{Thm:dual} says that the whole feature space is dual to decision
functions of the form $k(\mX,y_i)$, with $y_i$ generic/random; so, $k(X,Y)$ is the
proper object which captures both features of the interpolation space - through the
part of~$Y$ that is kernel orthogonal to $X$ - and the intrinsic features of the
data in~$X$ which can be obtained through $k(X,X)$ and are recovered e.g.~by kernel PCA.

\subsection{The Kernel Powering Method}\label{sec:kernels.power}
The IPCA algorithm (Alg.~\ref{Alg:features}) provides an estimate for the interpolation
space, and therefore the generative manifold~$\calS$, as discussed above.
However, there are two points where improvement is possible: (a) the features
learnt are all of the same degree, while there may be features of different degrees.
In particular, if an overly high~$d$ is chosen in IPCA, and if~$\calS$ is for
instance linear, this will not be explicitly noticed. (b) The size of the approximate
kernel border basis, i.e., the number of generators for $\Id (\calS)$,
when naively estimated as generators for the orthogonal of $\K [\calS]_{\le d}$,
grows exponentially in~$d$, since $\dim \K [\mX]_{\le d}$ does.\\
In the following we address these issues simultaneously by a
powering-projection-strategy applied to the kernel matrix. To address (a), we
increase degrees and learn features of increasing degrees step-by-step by
computing entrywise-powers of the degree~$1$ kernel matrix,
exploiting the fact that the degree~$d$ kernel matrix is the $d$-th power
To address (b), we use the absorption property of the ideal $\Id (\calS)$
to obtain a low number of multiplicative generators, by projecting the kernel
matrix onto a low rank approximation; by Theorem~\ref{Thm:dual-interp}~(iii)
this corresponds to adding new elements to the kernel border basis.
Powering is furthermore compatible with absorption by Equation~\ref{Eq:orthdeg}.
So, after powering, only new generators in higher degree will appear, and this
allows us to add them sequentially to the approximate kernel border basis.
Concretely, this works as follows. Start with the linear kernel matrix $K_1$.
Then project on a smaller rank matrix $K_1'$ by singular value thresholding.
Next, compute $K_2$ as the entrywise second power $K_2:=K_1'\otimes K_1$
and project again onto a smaller rank matrix $K_2'$.
In general, obtain $K_d:=K_{d-1}'\otimes K_1$, then threshold.
The threshold can be chosen fixed, or according to the Hilbert function of~$\calS$,
if that is known. In each step, we add features to the approximate kernel border
basis which correspond to singular values under the threshold, but not exactly zero.
This ``degree greedy'' strategy can be seen as a kernelization of some ideas in
the AVI class of algorithms~\cite{AVI,Sauer07}.

\subsection{Approximate Vanishing Ideal Component Analysis}
We now describe an algorithm, which uses the power-projecting strategy, called Approximate Vanishing Ideal Component Analysis (AVICA). AVICA will output generative and discriminative features of various degrees, ordered by informativity. As discussed in section~\ref{sec:kernels.power}, the main difference to IPCA (Alg.~\ref{Alg:features}) lies in the fact that the ``degree greedy'' strategy collects generators for the ideal of the manifold $\calS$ with increasing degree, therefore offers a much sparser generative description of $\calS$ than IPCA, while learning degree-ordered generators of the interpolation space as well. We present AVICA as Algorithm~\ref{Alg:AVICA}; for simplicity of reading, we first introduce notation for the projection step which is singular value thresholding:

\begin{Def}
For a matrix $K$ and a threshold $\epsilon$, we define the \emph{thresholded SVD} to be
$K = U S V + U^\perp S^\perp V^\perp,$
where the concatenations $(U,U^\perp)$ and $(V,V^\perp)^\top$ are the left and right singular matrices of the usual singular value decomposition, with singular values in $S$ having absolute value $\ge \epsilon$, and those in $S^\perp$ being $< \epsilon$.
\end{Def}

Applied to the kernel matrix, this means, according to Theorem~\ref{Thm:dual-interp}~(iii) that the features corresponding to $V^\perp$ are added to the approximate kernel border basis.

\begin{algorithm}[h]
\caption[\texttt{AVICA} Sparsely computes approximate interpolation space and kernel border basis]{\texttt{AVICA} Sparsely computes approximate interpolation space and kernel border basis.\\
\textit{Input:} data $x_1,\dots x_N\in \K^n$, given as rows of an $(N\times n)$ matrix $X$, maximum degree $maxdeg$, threshold $\epsilon$.\\
\textit{Output:} Feature functions $F_i = \sum_{j=1}^D \alpha_{ij} k(y_j,.)$, with labels "generative" or "discriminative" and quanta $q_i$. \label{Alg:AVICA}}
\begin{algorithmic}[1]
\STATE Sample random points $y_i\in \K^n, 1\le i\le D$; write those in a $(D \times n)$ matrix $Y$.	\label{Alg:AVICA.step1}
\STATE Let $K_0$ be the all-ones $(N \times D)$ matrix.
\STATE Compute the $N\times D$ matrix $K = k(X,Y)$.
\FOR{ $d = 1,\ldots, maxdeg$}
\STATE Set $\epsilon\leftarrow \epsilon\cdot \theta$
\STATE Set $K_d\leftarrow K_{d-1}\otimes K$ (entry-wise product)
\STATE Compute the $\epsilon$-thresholded SVD\\ $K_d = U_d S_d V_d + U^\perp_d S^\perp_d V^\perp_d$.
\STATE For each column $v$ of $V$, return a discriminative feature $F_*(.) = \sum_{j=1}^D v_j k^d (y_j,.)$.
\STATE For each column $v$ of $V^\perp$ with singular value that is not zero with machine precision, return a generative feature $F_*(.) = \sum_{j=1}^D v_j k^d (y_j,.)$
\STATE Also return as quantum $q_*$ the corresponding singular values times $\theta^d$.
\STATE $K_d\leftarrow U_d S_d V_d$
\ENDFOR
\STATE Display as informativity order the generative features ascendingly by $q_*$, the discriminative ones descendingly.
\end{algorithmic}
\end{algorithm}

Since AVICA computes a basis for the interpolation space in a similar way as IPCA (Alg.~\ref{Alg:features}), an analogous proof shows that AVICA is a noise consistent estimator for $\calS$ in the same sense. Evaluation of the features can again be done efficiently, by storing $Y,S_d,V_d,S^\perp_d,V^\perp_d$ as model parameters, then repeating the computations. Since this is sligtly more complex than in the case of IPCA, we describe this explicitly in form of Algorithm~\ref{Alg:AVICA_eval}.

\begin{algorithm}[h]
\caption[\texttt{eval.AVICA} Bulk evaluates AVICA features]{\texttt{ eval.AVICA }
Bulk evaluates features.\\
\textit{Input:} matrices $(Y,S_d,V_d,S^\perp, V^\perp)$, data $x_1,\dots x_N\in \K^n$, given as rows of an $(N\times n)$ matrix $X$, maximum degree $maxdeg$.\\
\textit{Output:} Evaluations of all feature functions\\ $F_i(x_k) = \sum_{j=1}^D \alpha_{ij} k(y_j,x_k)$. \label{Alg:AVICA_eval}}
\begin{algorithmic}[1]
\STATE Let $K_0$ be the all-ones $(N \times D)$ matrix.
\STATE Compute the $N\times D$ matrix $K = k(X,Y)$.
\FOR{ $d = 1,\ldots, maxdeg$}
\STATE Set $K_d\leftarrow K_{d-1}\otimes K$
\STATE Return $K_d\cdot V_d$ for discriminative and $K_d\cdot V_d^\perp$ for generative features. Rows are indexed by $k$, columns by $i$.
\STATE $K_d\leftarrow K_d V_d^\top V_d.$
\ENDFOR
\end{algorithmic}
\end{algorithm}

\subsection{AVICA for Discriminative Learning}\label{sec:kernels.disc}
From the discussion so far, it appears that the main advantage and novel of AVICA is learning generative features of some data manifold. However, with a minor but crucial modification, it can be adapted for discriminative supervised learning in a natural way which will allow one-vs-all or one-vs-one discrimination which is in some sense also class-generative. Namely, consider a feature $F(.)=\sum_{i=1}^D \alpha_i k(.,y_i)$ in the kernel border basis, that is, $F(s)\approx 0$ for $s\in \calS$. We have said that such an $F$ is generative, as it describes $\calS$ - but it can also be viewed discriminative, distinguishing $\calS$ from the ``set of general points'' in $\K^n$. While this is an unusual view, it is the one which generalizes well to discriminative learning: the $y_i$ were chosen to span $\K^n$ or the feature space; picking them, instead, as elements of a different manifold $\calS'\subsetneq \K^n$ will in the same way allow to distinguish $\calS$ from $\calS'$. Specifically, in step~\ref{Alg:AVICA.step1} of AVICA (Alg.~\ref{Alg:AVICA}), replace random sampling from $\K^n$ with random sampling in some $\calS'\supseteq \calS$, in order to learn to identify points in $\calS$ among points in $\calS'$. To learn a one-vs-all-classifier between classes $\calS_1,\dots, \calS_k$, choose $\calS = \calS_1$ and $\calS' = \calS_1\cup\dots\cup \calS_k$, then use the "generative" features, in the kernel border basis, as class discriminative.

\subsection{AVICA for Non-Polynomial Kernels}\label{sec:kernels.nonpoly}

We would like to stress that neither IPCA (Alg.~\ref{Alg:features}) nor AVICA (Alg.~\ref{Alg:AVICA}) makes a strong use of the polynomial kernel; for a general kernel, duality with a polynomial ring $\K [k(y_i,\mX),1\le i\le D]$ takes the place of duality with the polynomial ring $\K[\mX]$; the number $D$ has to be taken sufficiently large for the application. Again, generative and discriminative features can be both extracted with IPCA and AVICA. The interpolation space corresponds to the usual features learnt by kernel methods, while the manifold, or ideal, yields new generative ones, depending on the kernel. For example, when taking $k$ to be the Gauss kernel, AVICA will learn the \emph{clusters themselves}, instead of separators, since the clusters correspond to manifolds in the Gauss feature space.

\section{Experiments}\label{Sec:experiments}

\subsection{The circle}
A well-known example used to motivate discriminative kernel classification is a $2$-dimensional
problem, in which the classes are sampled from concentric circles.  %
\begin{figure*}[ht]
\begin{center}
\subfigure[$\sigma = 1.1$]{%
\includegraphics[height=0.12\textwidth]{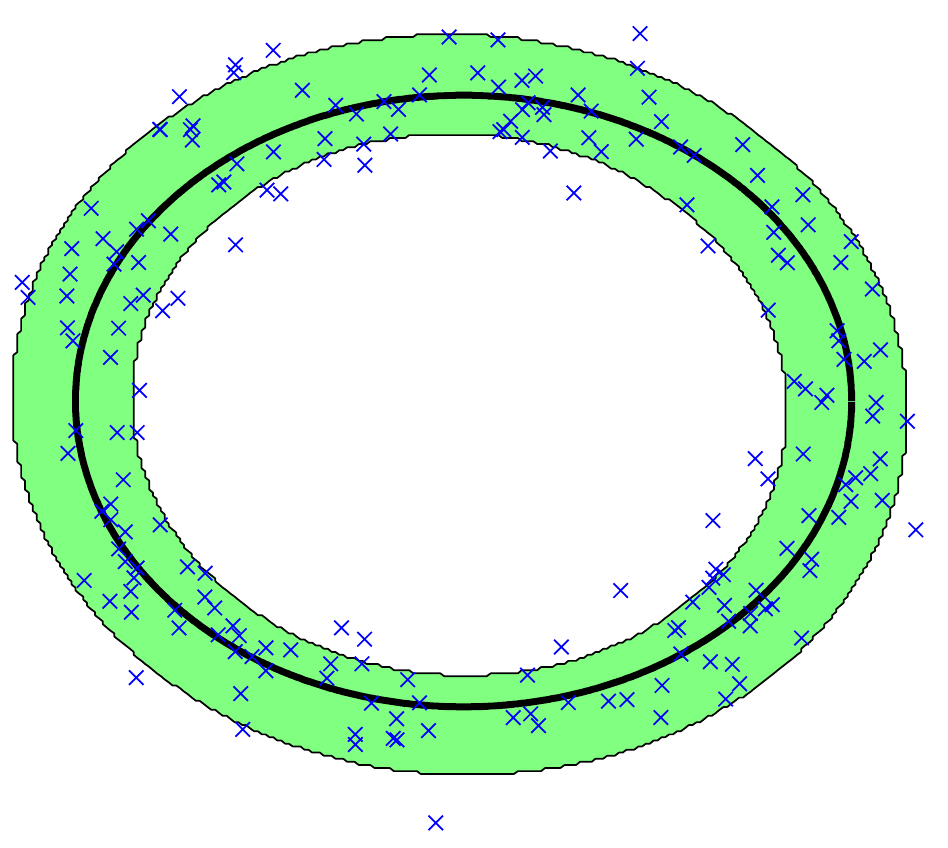}
}
\subfigure[$\sigma = 1.4$]{%
\includegraphics[height=0.12\textwidth]{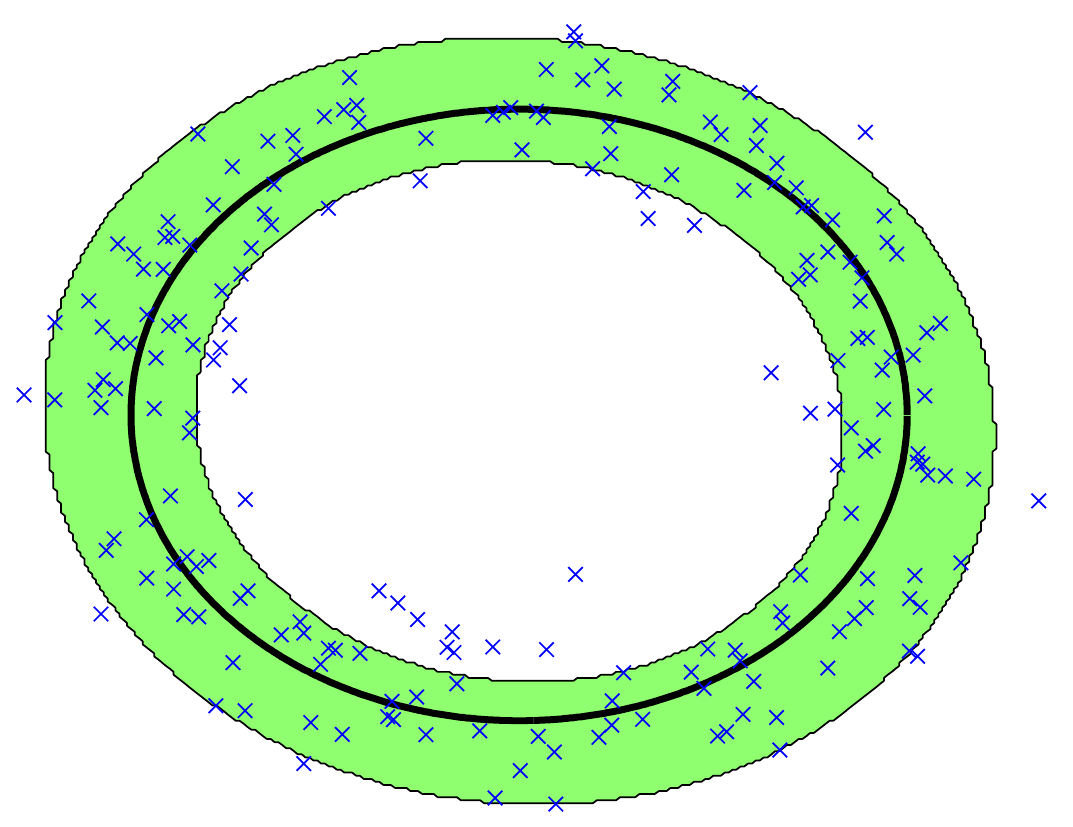}
}
\subfigure[$\sigma = 1.7$]{%
\includegraphics[height=0.12\textwidth]{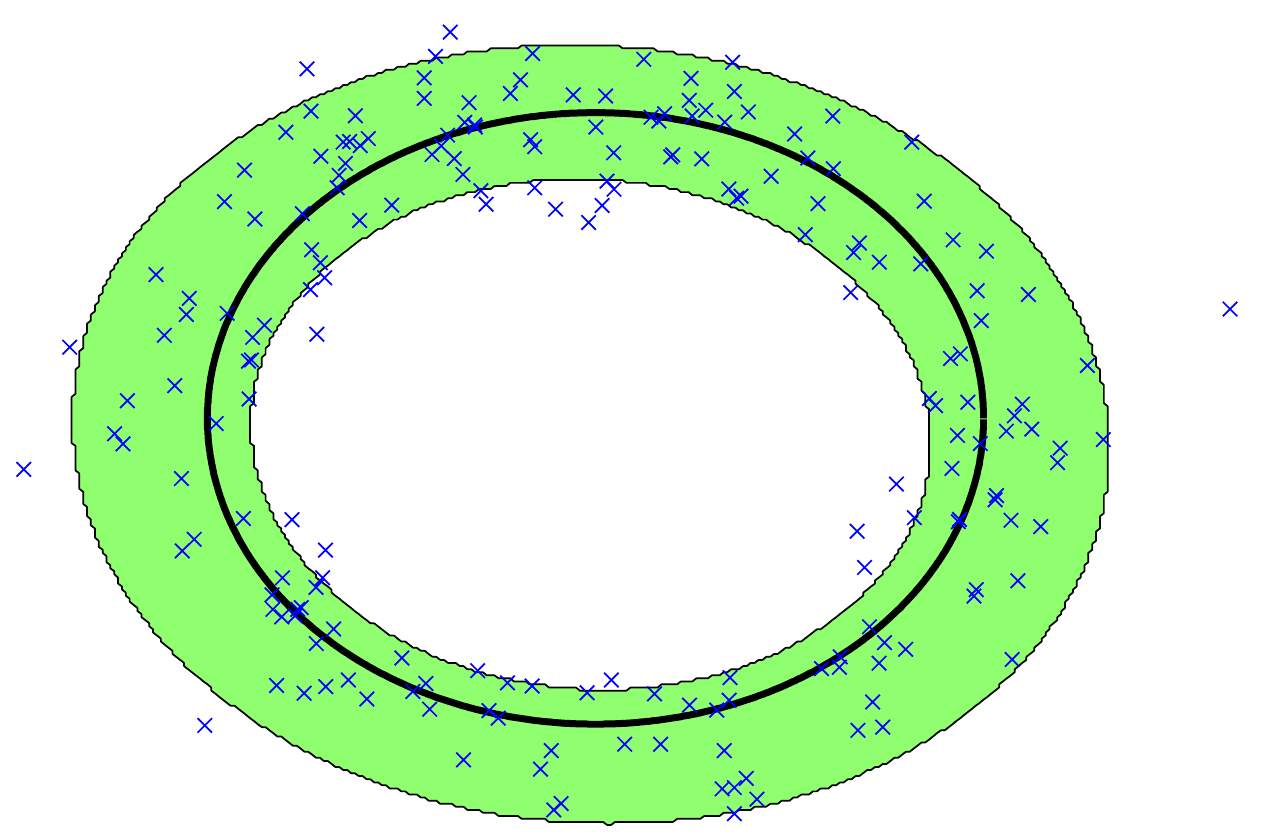}
}
\subfigure[$\sigma = 2.0$]{%
\includegraphics[height=0.12\textwidth]{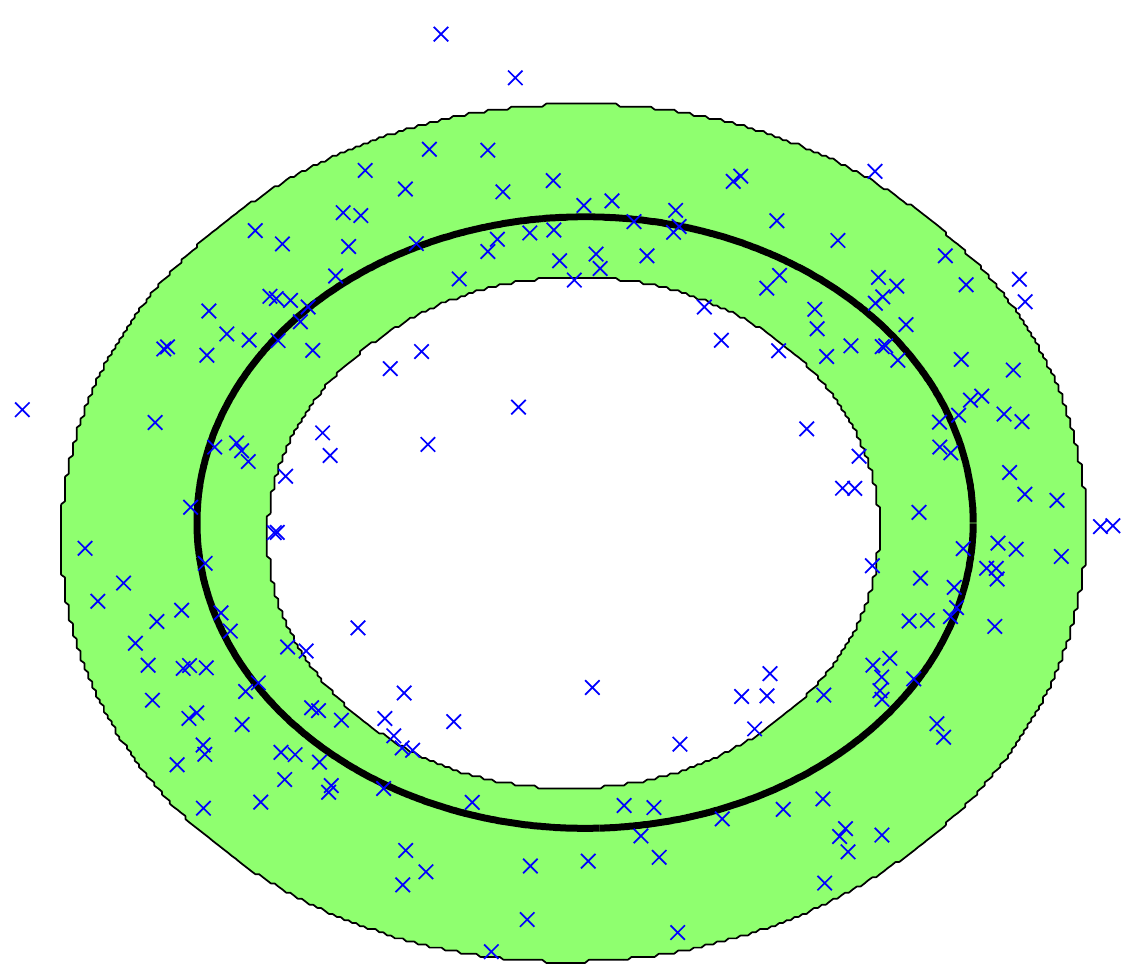}
}
\subfigure[$\sigma = 2.3$]{%
\includegraphics[height=0.12\textwidth]{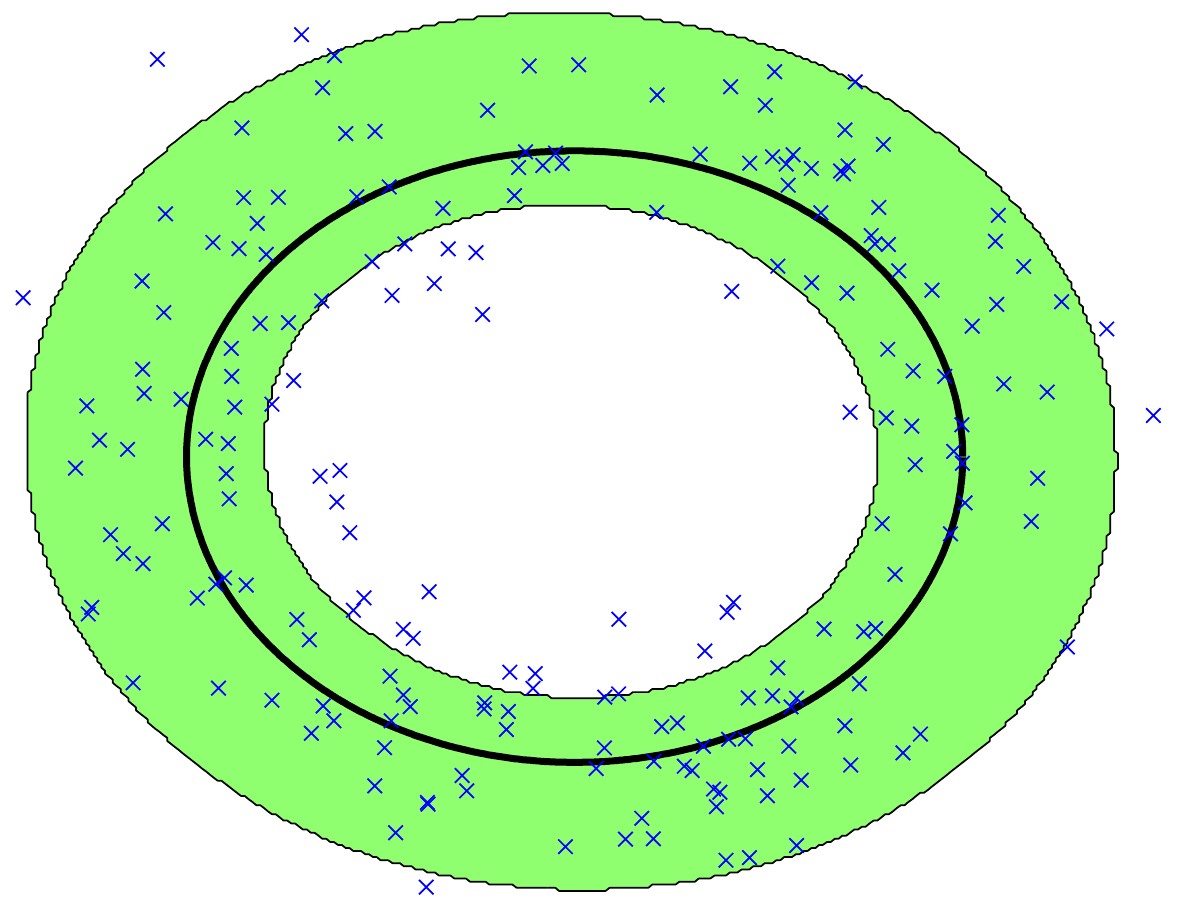}
}
\subfigure[$\sigma = 2.6$]{%
\includegraphics[height=0.12\textwidth]{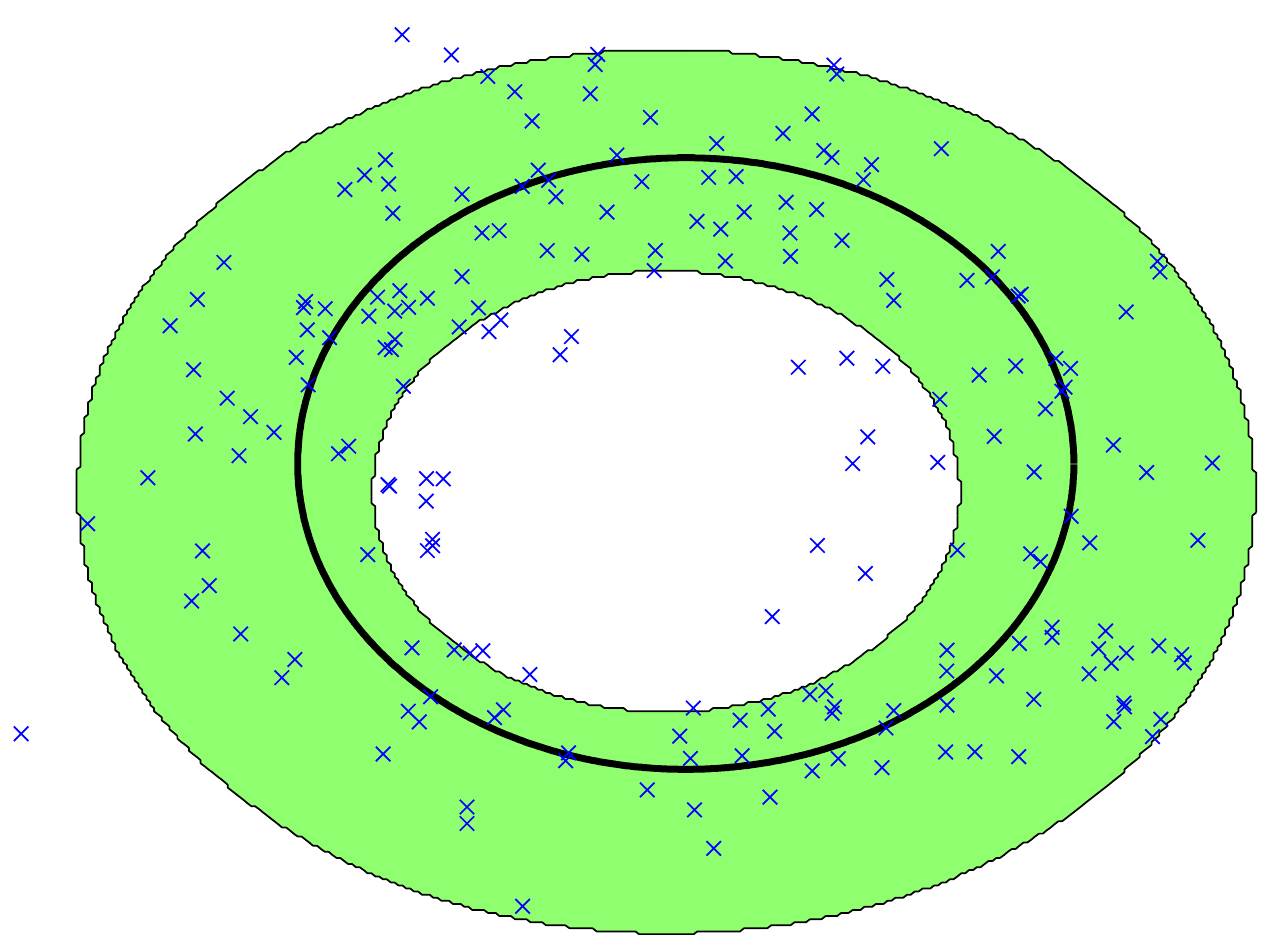}
}
\end{center}
\caption{%
\label{fig:circle}
Learning a circle of radius $10$ from $200$ random points with $N(0,\sigma)$
noise.  Blue x's are the training points.  The true circle is shown in black.
The green region is the $[-s/10,s/10]$ level set of the kernel border basis element with smallest singular value $s$.
}
\end{figure*}
The analogous task for generative kernel learning is learning \emph{one} circle from a noisy sample.
Figure \ref{fig:circle} shows that AVICA does this.  In each experiment, we generate $200$ uniform
points on a circle of radius $R=10$, centered at the origin, add $N(0,\sigma)$ and then run $\le 2$-AVICA with
threshold $2\theta\sqrt{2}\sigma^2$, which is what we expect for noise in the feature space.  The width
of the green region in Figure \ref{fig:circle} is scaled by $s/R$, since it is proportional to length of the
generator in feature space and captures the data. However, it is only there for illustrative purposes: the
estimation task is to estimate the \emph{manifold}, not the data, so what is important is not the width of
the green region but that the black circle is in it.

\subsection{Handwritten digits}
We tested AVICA on the MNist handwritten digit recognition data set, to compare it with the AVI class algorithm in~\cite{VCA2013}. Classification was done using the discriminative method described in section~\ref{sec:kernels.disc}; the union of all classes was subsampled to $200$ data points. The class to which a test point was assigned was chosen as the minimizer of the $\ell^1$-norm of the one-vs-all-generative features. As kernels, we used the inhomogenous kernel with $\theta=1/\sqrt{2}$, and the Gauss kernel with a width of $5000$. Thresholding in AVICA was done at the logarithmic mean of the singular value spectrum. For $maxdeg = 1$, that is with purely linear features (in which case IPCA=AVICA), both methods (polynomial and Gauss) achieved an overall misclassification rate of $4.1\%$ with an overall runtime in the order of seconds. Increasing the degree, the size of the subsample, or varying the parameters can lead to lower misclassification rates. However, it is difficult to compare these results to those of~\cite{VCA2013}, since the authors do not disclose how exactly they measure runtime or choose the threshold, therefore we refrain from more detailed comparison and conclude that IPCA and AVICA are already very fast and competitive on handwritten digits for degree $1$.

\section{Discussion}\label{Sec:discussion}

\subsection{Duality of Generative and Discriminative Learning}
Let us close out our theoretical discussion with two critical connections.
The first is between our method and the by-now classical theory of discriminative
learning with kernels. By our discussion on duality in section~\ref{Sec:ideals},
we have implicitly shown the following informal ``theorem'':

\begin{Thm}\label{Thm:learning-duality}
Generative learning with ideals is dual to discriminative learning with
kernels; discriminative learning with ideals is dual to generative learning
with kernels.
\end{Thm}

For example, in the classical discriminative scenario of the
\emph{kernel support vector machine}, the standard
kernel decision function is of the form $F(.)=\sum_{i=1}^D\alpha_i k(.,x_i)$.
In the ideals setting, this is generative learning of the separating hyperplane,
which is a manifold uniquely parametrized by~$F$, interpreted as a polynomial
and element of the \emph{interpolation space}. In \emph{kernel PCA},
generative features are learnt for the data; by using an analogous technique
in the dual, \emph{IPCA} learns the data manifold in a discriminative way,
by separating points on the manifold (the matrix~$X$) from points not on the
manifold (the matrix~$Y$). Moreover, the \emph{noise consistency} guarantees that we obtain for both IPCA
and AVICA are dual to \emph{generalization bounds} that can be obtained
from classical Vapnik-Cher\-vo\-nen\-kis theory.

\subsection{Related Work}

Finally, we briefly discuss the principal strains of related work and the ideas which are rooted there. These are, in the sequence we will discuss them: kernels, manifold learning, algebra in statistics and learning, approximate symbolic methods.
{\bf Kernel methods} form a broad field and have been widely studied in practical
and theoretical context; a detailed overview over the field and its history can,
for example, be found in the ``further reading'' sections of~\cite{Shawe-Taylor04}.
To our knowledge, there is so far no technique or result relating kernel methods to
symbolic computation. The major link to existing literature is through the kernel trick~\cite{Aizerman64} and the reproducing kernel hilbert space duality (see section 2.2.3 of~\cite{Scholkopf02}): the initial statements on algebra-kernel duality can be obtained from RKHS theory by considering polynomial functions formally as elements of the polynomial rings.
{\bf Manifold learning} techniques, such as principal curves~\cite{Hastie84}, LLE~\cite{LLE2000},
or the aforementioned kernel PCA~\cite{KPCA1998} have in common that they assign an embedding of the data into low dimensional space.
This corresponds to learning discriminative features. Our algorithms IPCA and AVICA also learn features which generatively describe the manifold, i.e., explicitly describe where the points lie in the high dimensional data space.
As the discussion in section~\ref{sec:kernels.analysis} explains in greater detail,
these methods can be seen as an extension of kernel PCA in the sense
that they do not only learn the embedding, but also the manifold.
{\bf Algebraic techniques in statistics} have been a recurring topic since the
advent of algebraic statistics, for an overview see~\cite{Pis02,Stu02,Gib10,Stu10}.
However, the results in algebraic statistics are not directly applicable to a
learning or data related context, since the field is predominantly concerned
with understanding algebraically structured models and not estimating them from data.
On the other hand, there are seemingly unrelated scenarios where specific algebraic
structures have explicitly been used for estimation and learning in particular
scenarios, e.g.~\cite{GPCA05,Kondor08,Kir12JMLR,KT2013NIPS}.
A learning theory built on polynomials and ideals has been outlined in the
appendix of~\cite{Kir12AISTATS}.
{\bf Approximate symbolic computation} techniques can be traced back to Corless et al~\cite{Cor95} who proposed the use of singular value decomposition (SVD) for polynomial systems, and the work of Stetter~\cite{Ste04} who pioneered a more
general numerical view. The first algorithms which use SVD to estimate an approximate
vanishing ideal numerically are those of Heldt et al~\cite{AVI}, which uses
border bases and a numerically stable variant of term orderings, and Sauer~\cite{Sauer07}, which uses a coordinate independent degree-increasing strategy to compute homogenous bases; both algorithms can be considered as variations on the same idea set and yield
an approximate version of the exact symbolic Buchberger-Möller algorithm~\cite{Moller82}.
The homogenous variant in~\cite{Sauer07} has recently reappeared under the
name ``Vanishing Components Analysis''~\cite{VCA2013} in the machine
learning community. {\bf AVICA} can be seen as a kernelization of these AVI-class
algorithms: it integrates the idea of the compact order ideal/border basis
representation in~\cite{AVI} (interpolation space/kernel border basis) with
the degree-greedy strategy and homogenous coordinate independence of~\cite{Sauer07}
into our kernel based algorithm through the concept of kernel-ideal-duality.

\subsection{Conclusion and Outlook}
In this paper, we have exposed an intricate duality between kernels and commutative
algebra, between ideals and manifolds, between kernel methods and symbolic algebraic
methods, between generative and discriminative learning. We have outlined how a
statistical learning theory in this new ideal-kernel-duality setting can be obtained,
and we have described how the duality can be exploited in general for learning
explicit, generative structures with kernels. We have demonstrated, theoretically
and competitively on real-world-data, how the duality can be used to construct
a novel type of algorithm, AVICA, which simultaneously extracts discriminative
and generative components from the data.\\
In the light of this, the whole field of statistical data analysis and machine
learning stands open to a plethora of new methods following this conceptual regime.
\section*{Acknowledgments}
LT is supported by the European Research Council under the European Union’s Seventh Framework
Programme (FP7/2007-2013) / ERC grant agreement no 247029- SDModels.  This research was
carried out at MFO, supported by FK's Oberwolfach Leibniz Fellowship.

\bibliographystyle{plainnat}

\end{document}